\documentclass{cstr}

\usepackage{amsmath,amsfonts,amssymb,latexsym,bm}
\usepackage{amsthm}
\usepackage{mathrsfs}
\usepackage[dvipdfmx]{graphicx}
\usepackage{subfig}
\usepackage{url}
\usepackage{cite}

\usepackage{booktabs}
\usepackage{color}

\newtheorem{theorem}{Theorem}

\usepackage{algorithm, algorithmic}

\title{
Accuracy and Privacy Evaluations of Collaborative Data Analysis
}

\author[1,*]{Akira Imakura}
\author[1]{Anna Bogdanova}
\author[1]{Takaya Yamazoe}
\author[1]{Kazumasa Omote}
\author[1]{Tetsuya Sakurai}

\affil[1]{University of Tsukuba, 1-1-1 Tennodai, Ibaraki, Tsukuba 305-8573, Japan}

\email{imakura@cs.tsukuba.ac.jp}

\begin{document}
\maketitle
\thispagestyle{titlepage}

\begin{abstract}
Distributed data analysis without revealing the individual data has recently attracted significant attention in several applications.
A collaborative data analysis through sharing dimensionality reduced representations of data has been proposed as a non-model sharing-type federated learning.
This paper analyzes the accuracy and privacy evaluations of this novel framework.
In the accuracy analysis, we provided sufficient conditions for the equivalence of the collaborative data analysis and the centralized analysis with dimensionality reduction.
In the privacy analysis, we proved that collaborative users' private datasets are protected with a double privacy layer against insider and external attacking scenarios.
\end{abstract}

\section{Introduction}
\subsection{Background}
Recently, the problem of real-life data availability for machine learning and data analysis applications came to the forefront of actual research challenges.
In particular, use-cases that pertain to sensitive personal information or corporate secrecy can benefit from the ability to process distributed data without revealing it to other parties.
\par
Various methods have been proposed over recent years, involving sharing a machine learning model that is collectively trained among several parties.
In the present work, we analyze an alternative method of distributed and privacy-preserving data analysis that does not require sharing the machine learning model.
The non-model-sharing approach has certain advantages over model-sharing methods: 
(a) maintaining the secrecy of particular model architecture; 
(b) protection from model poisoning attacks; 
(c) avoiding iterative communications necessary for machine learning training; 
(d) option to outsource data analysis to a third party or a data analysis competition. 
\par
The collaborative data analysis considered in this paper had been previously proposed in \cite{imakura2020data,imakura2020collaborative}.
However, the method's proper accuracy and privacy analysis were lacking.
In present work, we fill in this gap by providing conditions for equivalence of the data analysis in centralized and distributed settings, as well as conducting a thorough privacy analysis and disclosure risk evaluation of the collaborative data analysis.
\subsection{Main purposes and contributions}
The accuracy and privacy analyses are essential in practical use of the collaborative data analysis.
In this paper, we analyze the equivalence of the collaborative data analysis and the centralized analysis with dimensionality reduction.
We also analyze the privacy of the private dataset against insider and external attacking scenarios.
\par
The main contributions of this paper are summarized as follows:
\begin{itemize}
  \item We provided the sufficient condition for equivalence of the collaborative data analysis and the centralized analysis with dimensionality reduction.
  \item We proved that, in the collaborative data analysis, the private dataset is protected based on a double privacy layer against insider and external attacking scenarios.
  \item We demonstrated numerical evaluations for the accuracy and privacy analyses.
\end{itemize}
\subsection{Related Work}
The problem of deriving insights from data while maintaining the privacy of individual data records was first addressed in the context of Data Mining and Knowledge Discovery in Databases (KDD) \cite{agrawal2000privacy} and consequently formed a large body of literature known as Privacy-Preserving Data Mining (PPDM).
This field studies data-sanitizing operations, which can offer quantifiable privacy guarantees, while maintaining data utility for a variety of downstream analytical tasks, including supervised and unsupervised machine learning \cite{mendes2017privacy}. 
\par
As there are multiple definitions of what constitutes privacy and how it should be measured, distinct privacy guarantees and methods of privacy production became known in the literature as privacy models.
Most influential privacy models that emerged from PPDM are $k$-anonymity, proposed by Samarati and Sweeny \cite{samarati1998protecting}, and $\varepsilon$-differential privacy, introduced by Dwork \cite{dwork2008differential}.
$K$-anonymity protects users' data from linkage attacks, ensuring that released data has at least $k$ identical records.
$\varepsilon$-differential privacy, on the other hand, guarantees that the inclusion of any record in the dataset will not change the output of data mining operations by more than $\varepsilon$, thus preventing membership inference attacks.
Both privacy models are theoretically sound, deployed in practice, and legislatively recognized.
However, there are significant shortcomings that call for the development of alternative notions of privacy.
Thus, $k$-anonymity is proven to be NP-hard \cite{verykios2004state} and not attainable on high dimensional and sparse datasets \cite{narayanan2006break}.
Similarly, providing record-level $\varepsilon$-differential privacy is not suitable for modern deep learning applications \cite{zhao2019differential}, as the amount of perturbation required diminishes data utility. 
\par
With the advancement of highly parameterized machine learning models, the focus of data privacy research shifted towards designing model architectures and sanitizing model parameters to enable Privacy-Preserving Machine Learning (PPML).
One particularly successful approach had been Federated Learning proposed by McMahan et al. \cite{mcmahan2017communication}.
It is a machine learning framework that allows distributed training of deep learning models through the averaging of gradient descent steps taken on private datasets.
Analogous algorithms were introduced for numerous machine learning models and various distributed settings, forming what became known as Federated Learning Systems \cite{li2019federated}.
Although currently, PPML cannot provide a formal privacy guarantee, as PPDM does, it satisfies privacy requirements through the data minimization approach, by sharing only the information necessary to the particular analytical task \cite{kairouz2019advances}.
Additionally, PPML is often combined with encryption schemes to prevent inference from intermediary results of the computation.
\par
The collaborative data analysis \cite{imakura2020data,imakura2020collaborative}, the method considered in this paper, is positioned in between the two approaches to privacy-preserving data analysis.
It shares with PPDM the focus on data transformation and the release of sample-vise information, which can be further explored for hidden relations and patterns.
At the same time, it employs the information-minimization approach of PPML through the dimensionality reduction operation on the original data.
Moreover, the shared intermediate representations can be formed by hidden layers of deep neural networks, strongly relating our method to the PPML domain.
\par
Privacy-preserving properties of dimensionality reduction were previously explored in several papers.
Thus, Tai et el. \cite{tai2018exploring}, demonstrated that dimensionality reduction on average increases the value of $k$ in $k$-anonymity privacy model, although it does not perform it reliably.
Similarly, Nguyen and colleagues \cite{nguyen2020autogan} developed a $\varepsilon$-DR privacy framework of measuring the information loss of dimensionality reduction  operations analogous to $\varepsilon$-differential privacy.
Formal privacy guaranteed were demonstrated for particular methods of dimensionality reduction, such as non-metric multidimensional scaling (MDS) \cite{alotaibi2012non} and random projections \cite{liu2005random}.
Moreover, specialized methods of dimensionality reduction were developed to satisfy certain privacy models, for instance differential-private Principal Component Analysis, and differential-private Linear Discriminant Analysis \cite{jiang2013differential}.
Since Data Collaboration method assumes an arbitrary dimensionality reduction function applied at the user's side, in practical applications such methods can be chosen to satisfy necessary privacy standards.
\par
To the best of our knowledge, Data Collaboration is the only method so far offering collaborative data analysis through sharing dimensionality reduced representations of data and integration of such representations in a unified subspace.
In this work, for the first time we propose privacy guarantees as well as utility measures of the transformed collaborative representations of data.
\section{Collaborative data analysis}
\subsection{Distributed data analysis}
Let $m$ and $n$ denote the numbers of features and training data samples.
In addition, let $X = [{\bm x}_{1}, {\bm x}_{2}, \dots, {\bm x}_{n}]^{\rm T} \in \mathbb{R}^{n \times m}$ and $Y = [{\bm y}_1, {\bm y}_2, \dots, {\bm y}_n]^{\rm T} \in \mathbb{R}^{n \times \ell}$ be the training dataset and the corresponding ground truth.
The $n$ data samples are partitioned into $c$ parties as follows:
\begin{equation}
  X = \left[
    \begin{array}{c}
      X_{1} \\
      X_{2} \\
      \vdots  \\
      X_{c} 
    \end{array}
  \right], \quad
  Y = \left[
    \begin{array}{c}
      Y_{1} \\
      Y_{2} \\
      \vdots \\
      Y_{c} 
    \end{array}
  \right].
\end{equation}
Then, the $i$-th party has partial dataset and the corresponding ground truth,
\begin{equation*}
  X_{i} \in \mathbb{R}^{n_i \times m}, \quad Y_i \in \mathbb{R}^{n_i \times \ell}.
\end{equation*}
\par
A motivating example could be found in distributed medical data analysis.
An analysis only using the dataset in each medical institution, i.e., {\it individual analysis} may not be sufficient for generating a high-quality prediction result due to insufficiency and imbalance of the data samples.
If we can centralize the datasets from multiple institutions and analyze them as one dataset, i.e., {\it centralized analysis}, then we expect to achieve a high-quality prediction.
However, it is difficult to centralise the original medical data samples with those from other institutions due to confidentiality concerns.
Such kind of distributed data analysis is also essential in other applications, e.g., financial and manufacturing data analysis.
\subsection{Outline of the collaborative data analysis}
The collaborative data analysis has been proposed in \cite{imakura2020data,imakura2020collaborative} as a method of distributed data analysis.
A practical operation strategy regarding privacy and confidentiality concerns is also introduced.
Here, we briefly introduce the algorithm based on the practical operation strategy.
\par
In the practical operation strategy, the collaborative data analysis is operated by two roles: {\it user} and {\it analyst}.
The users have the private dataset $X_{i}$ and the corresponding ground truth $Y_i$ and want to analyze them without sharing $X_{i}$.
Each user individually constructs a dimensionality reduced intermediate representation and centralize it to analyst.
To allow each user to use individual function for generating the intermediate representation, analyst transforms again the centralized intermediate representations to an incorporable form called {\it collaboration representations}.
For constructing the incorporable collaboration representations, users generate a shareable {\it anchor dataset} and centralize its intermediate representation to analyst.
Then, the collaborative representation is analyzed as one dataset.
\subsubsection{Training phase}
First, all users generate the same anchor dataset $X^{\rm anc} \in \mathbb{R}^{r \times m}$, which is a shareable data consisting of public data or dummy data randomly constructed, and partition it by features.
Then, each user constructs the intermediate representations,
\begin{align*}
  \widetilde{X}_{i} = f_{i}(X_{i}) \in \mathbb{R}^{n_i \times \widetilde{m}_{i}}, \quad
  \widetilde{X}_{i}^{\rm anc} = f_{i}(X^{\rm anc}) \in \mathbb{R}^{r \times \widetilde{m}_{i}},
\end{align*}
where $f_{i}$ denotes a linear or nonlinear row-wise mapping function and centralize the intermediate representations to the analyst.
A typical setting for $f_{i}$ is a dimensionality reduction, with $\widetilde{m}_{i} < m$, including unsupervised methods \cite{pearson1901liii,he2004locality,maaten2008visualizing} and supervised methods \cite{fisher1936use,sugiyama2007dimensionality,li2017locality,imakura2019complex}.
For privacy and confidentiality concerns, the function $f_{i}$ should be set as
\begin{itemize}
  \item The private data $X_{i}$ can be obtained only if anyone has both the corresponding intermediate representation $\widetilde{X}_{i}$ and the mapping function $f_{i}$ or its approximation.
  \item The mapping function $f_{i}$ can be inferred only if anyone has both the input and output of $f_{i}$.
\end{itemize}
\par
At the analyst side, the mapping function $g_i$ for the collaboration representation is constructed satisfying
\begin{equation*}
  \widehat{X}_i^{\rm anc} = g_i(\widetilde{X}_i^{\rm anc}) \in \mathbb{R}^{r \times \widehat{m}}
        \quad \mbox{s.t. }
        \widehat{X}_{i}^{\rm anc} \approx \widehat{X}_{i'}^{\rm anc} \quad
        (i \neq i'),
\end{equation*}
in some sense.
For computing $g_i$, authors of \cite{imakura2020data,imakura2020collaborative} introduced a practical method via a total least squares problem when $g_i$ is linear and also indicated an idea when $g_i$ is nonlinear.
\par
Then, the obtained collaboration representations $\widehat{X}_i = g_i(\widetilde{X}_i)$ can be analyzed as one dataset,
\begin{equation*}
  \widehat{X} = [\widehat{X}_1^{\rm T}, \widehat{X}_2^{\rm T}, \dots, \widehat{X}_c^{\rm T}]^{\rm T} \in \mathbb{R}^{n \times \widehat{m}},
\end{equation*}
with the shared ground truth $Y_i$ using some supervised machine learning and the deep learning methods.
The functions $g_i$ and $h$ are returned to the $i$-th user.
\subsubsection{Prediction Phase}
\label{sec:prediction}
%
%
Let $X_i^{\rm test} \in \mathbb{R}^{s_i \times m}$ be a test dataset of the $i$-th party.
Then, for prediction phase, the predictive result $Y_i^{\rm test}$ of $X_i^{\rm test}$ is obtained by
\begin{equation*}
Y_i^{\rm test} = h( g_i(f_i (X_{i}^{\rm test})))
\end{equation*}
via the intermediate and collaboration representations.
\begin{algorithm*}[!t]
\caption{Collaborative data analysis}
\label{alg:proposed}
\begin{algorithmic}
  \REQUIRE $X_{i} \in \mathbb{R}^{n_i \times m}, Y_i \in \mathbb{R}^{n_i \times \ell}, X_{i}^{\rm test}$, individually
  \ENSURE $Y_i^{\rm test}$ $(i = 1, 2, \dots, c)$.
  \STATE
  \STATE
  \begin{tabular}{rlcl}
    & \multicolumn{1}{c}{ {\it user side} $(i)$} && \multicolumn{1}{c}{ {\it analyst side} } \\ \cmidrule{2-2} \cmidrule{4-4}
    & \multicolumn{3}{c}{ ---------- Training phase ----------} \\
    1:  & Generate $X^{\rm anc}_{i}$ and share to all users && \\
    2:  & Set $X^{\rm anc}$ && \\
    3:  & Generate $f_{i}$ && \\
    4:  & Compute $\widetilde{X}_{i} = f_{i}(X_{i})$ and $\widetilde{X}^{\rm anc}_{i} = f_{i}(X^{\rm anc})$ && \\
    5:  & Share $\widetilde{X}_{i}, \widetilde{X}_{i}^{\rm anc}$ and $Y_i$ to analyst & $\rightarrow$ & Get $\widetilde{X}_{i}, \widetilde{X}_{i}^{\rm anc}$ and $Y_i$ for all $i$ \\
    6:  & && Construct $g_i$ from $\widetilde{X}_{i}^{\rm anc}$ for all $i$ \\
    7:  & && Compute $\widehat{X}_{i} = g_i(\widetilde{X}_{i})$ for all $i$ \\
    8:  & && Set $\widehat{X}$ and $Y$ \\
    9:  & && Analyze $\widehat{X}$ and get $h$ as $Y \approx h(\widehat{X})$ \\
    10: & Get $g_i$ and $h$ &$\leftarrow$& Return $g_i$ and $h$ to user \\
    \\
    & \multicolumn{3}{c}{ ---------- Prediction phase ----------} \\
    11: & Compute $Y_i^{\rm test} = h(g_i(f_i(X_i^{\rm test})))$ && \\
  \end{tabular}
\end{algorithmic}
\end{algorithm*}
\section{Accuracy analysis}
We analyze the accuracy of the collaborative data analysis for the simple case that the mapping functions $f_{i}$ and $g_i$ are linear, i.e.,
\begin{align*}
  &\widetilde{X}_{i} = f_{i}(X_{i}) = X_{i} F_{i}, \quad F_{i} \in \mathbb{R}^{m \times \widetilde{m}} \quad ({\rm rank}(F_{i}) = \widetilde{m}), \\
  &\widehat{X}_i = g_i(\widetilde{X}_i) = \widetilde{X}_i G_i, \quad G_i \in \mathbb{R}^{\widetilde{m} \times \widetilde{m}} \quad ({\rm rank}(G_i) = \widetilde{m}).
\end{align*}
Here, for simplicity, we assume that the dimensionality of $\widetilde{X}_{i}$ does not depend on $i$.
Also, the matrices $G_i$ are computed as introduced in \cite{imakura2020data,imakura2020collaborative}, that is, 
\begin{equation}
  \min_{G_i \in \mathbb{R}^{\widetilde{m}_i \times \widehat{m}} } \sum_{i=1}^c \| Z - \widetilde{X}_i^{\rm anc} G_i \|_{\rm F}^2,
  \label{eq:norm}
\end{equation}
where $Z \in \mathbb{R}^{r \times \widetilde{m}}$ is set as a column orthogonal matrix whose columns are the left singular vectors corresponding to the $\widetilde{m}$ largest singular values of a matrix
\begin{equation*}
  [\widetilde{X}^{\rm anc}_1, \widetilde{X}^{\rm anc}_2, \dots, \widetilde{X}^{\rm anc}_c] = X^{\rm anc}[F_1, F_2, \dots, F_c].
\end{equation*}
\subsection{Theoretical evaluation for accuracy}
In this paper, we analyze the accuracy of the collaborative data analysis compared with the centralized analysis with dimensionality reduction $B \in \mathbb{R}^{m \times \widetilde{m}}$ based on the norm
\begin{align}
  &\left| \! \left| X B - \left[ \begin{array}{c}
      X_1 F_1 G_1 \\
      X_2 F_2 G_2 \\
      \vdots \\
      X_c F_c G_c
    \end{array}
  \right] \right| \! \right|_{\rm F}^2 / \| X \|_{\rm F}^2 \nonumber \\
  &\quad = \sum_{i=1}^c \| X_i B - X_i F_i G_i \|_{\rm F}^2 / \| X \|_{\rm F}^2.
  \label{eq:error}
\end{align}
With the anchor dataset $X^{\rm anc}$ preserving statistics of $X$, we evaluate the accuracy \eqref{eq:error} by
\begin{align*}
  &\sum_{i=1}^c \| X_i B - X_i F_i G_i \|_{\rm F}^2 / \| X \|_{\rm F}^2 \\
  &\quad \approx \sum_{i=1}^c \| X^{\rm anc} B - X^{\rm anc} F_i G_i \|_{\rm F}^2 / (c \| X^{\rm anc} \|_{\rm F}^2) \\
  &\quad \leq \frac{ \sum_{i=1}^c \| X^{\rm anc}B - Z \|_{\rm F}^2  +  \| Z - X^{\rm anc} F_i G_i \|_{\rm F}^2}{ c \| X^{\rm anc} \|_{\rm F}^2}.
\end{align*}
\par
Under the assumption that $f_{i}$ are linear, we have $\widetilde{X}_i = X_i F_i$, where $X_i \in \mathbb{R}^{n_i \times m}$ and $F_i \in \mathbb{R}^{m \times \widetilde{m}}$.
Let $F = [F_1, F_2, \dots,F_c]$ and
\begin{align*}
  & X^{\rm anc} F = U \Sigma V^{\rm T} = [U_1, U_2] \left[
    \begin{array}{cc}
      \Sigma_1 & \\
      & \Sigma_2 
    \end{array}
  \right] \left[
    \begin{array}{c}
      V_1^{\rm T} \\
      V_2^{\rm T}
    \end{array}
  \right], \\
  & F = U_F \Sigma_F V_F^{\rm T} = [U_{F1}, U_{F2}] \left[
    \begin{array}{cc}
      \Sigma_{F1} & \\
      & \Sigma_{F2} 
    \end{array}
  \right] \left[
    \begin{array}{c}
      V_{F1}^{\rm T} \\
      V_{F2}^{\rm T}
    \end{array}
  \right]
\end{align*}
be singular value decompositions of matrices $X^{\rm anc}F$ and $F$.
Here, $\Sigma_1, \Sigma_{F1} \in \mathbb{R}^{\widetilde{m} \times \widetilde{m}}$ are the diagonal matrices corresponding to $\widetilde{m}$ largest singular values.
Note that $Z = U_1$.
Then, we have
\begin{align*}
  \| \Sigma_2 \|_{\rm F}^2 
  & = \min_{ {\rm rank}(\widetilde{X}) = \widetilde{m}}  \| X^{\rm anc} F - \widetilde{X} \|_{\rm F}^2 \\
  & \leq \| X^{\rm anc}  F - X^{\rm anc} U_{F1} \Sigma_{F1} V_{F1}^{\rm T} \|_{\rm F}^2 \\
  & \leq \| X^{\rm anc}  \|_{\rm F}^2 \| F - U_{F1} \Sigma_{F1} V_{F1}^{\rm T} \|_{\rm F}^2 \\
  & = \| X^{\rm anc} \|_{\rm F}^2 \| U_{F2} \Sigma_{F2} V_{F2}^{\rm T} \|_{\rm F}^2 \\
  & = \| X^{\rm anc} \|_{\rm F}^2 \| \Sigma_{F2} \|_{\rm F}^2.
\end{align*}
Using this inequality, the norm \eqref{eq:norm} can be bounded by
\begin{align*}
  & \min_{G_i} \sum_{i=1}^c \| Z - {X}^{\rm anc} F_i G_i \|_{\rm F}^2  \nonumber \\
  & \quad = \sum_{i=1}^c \min_{G_i} \| Z - {X}^{\rm anc} F_i G_i \|_{\rm F}^2  \nonumber \\
  & \quad \leq \sum_{i=1}^c \| G_i \|_{\rm F}^2 \min_{G_i^{-1}} \| Z G_i^{-1}  - {X}^{\rm anc} F_i \|_{\rm F}^2 \nonumber \\
  & \quad \leq \left( \max_i \| G_i \|_{\rm F}^2 \right) \min_{G_i^{-1}} \| Z [G_1^{-1}, G_2^{-1}, \dots, G_c^{-1}] \nonumber \\
  & \quad \hphantom{\leq \left( \max_i \| G_i \|_{\rm F}^2 \right) \min_{G_i^{-1}}\| } \quad - {X}^{\rm anc} [F_1, F_2, \dots, F_c] \|_{\rm F}^2 \nonumber \\
  & \quad = \left( \max_i \| G_i \|_{\rm F}^2 \right) \| \Sigma_2 \|_{\rm F}^2 \nonumber \\
  & \quad \leq \left( \max_i \| G_i \|_{\rm F}^2 \right) \| X^{\rm anc} \|_{\rm F}^2 \| \Sigma_{F2} \|_{\rm F}^2.
  \label{eq:bound}
\end{align*}
Therefore, the accuracy of the collaborative data analysis \eqref{eq:error} can be evaluated by $\| \Sigma_{2} \|_{\rm F}^2$ and $\| \Sigma_{F2} \|_{\rm F}^2$.
Note that the value $\| \Sigma_2 \|_{\rm F}^2$ can be obtained at the analyst side.
\par
This bound provides the following theorem.
\begin{theorem}
If the mapping functions satisfy 
\begin{equation}
  \mathcal{R}(F_1) = \mathcal{R}(F_2) = \dots = \mathcal{R}(F_c), \quad {\rm rank}(X^{\rm anc}F_i) = \widetilde{m},
  \label{eq:cond}
\end{equation}
the predictive results of the collaborative data analysis is mathematically equivalent to that of the centralized analysis with dimensionality reduction $F_1 G_1$.
\end{theorem}
\begin{proof}
The condition \eqref{eq:cond} provides $\Sigma_{F2}=O$, then we have
\begin{equation*}
  F_1 G_1 = F_2 G_2 = \dots = F_c G_c,
\end{equation*}
that proves the theorem.
\end{proof}
%
%
\subsection{Numerical evaluation for accuracy}
Here, we provide numerical evaluation of the accuracy analysis.
We used a 10-class classification of handwritten digits (MNIST) \cite{lecun1998mnist}, where $m=784$.
We set the number of parties as $c=4$ and the number of samples for each party as $n_i = 50$.
\par
Let $B \in \mathbb{R}^{784 \times 25}$ be a mapping function generated by PCA using all the training dataset $X$.
We then set
\begin{equation*}
  F_i = B E_i^{(1)} + \varepsilon \| B \|_{\rm F} E_i^{(2)}, \quad
  i = 1, 2, \dots, c,
\end{equation*}
with $E_i^{(1)} \in \mathbb{R}^{25 \times 25}$ and $E_i^{(2)} \in \mathbb{R}^{784 \times 25}$ whose entries are normally distributed random numbers.
We used a kernel version of ridge regression (K-RR) \cite{saunders1998ridge} with a Gaussian kernel for analyzing the collaboration representation.
The bandwidth $\sigma$ of the Gaussian kernel is set based on the local scaling \cite{zelnik2005self}.
We set the regularization parameter for K-RR to $\lambda = 0.1$.
The anchor data $X^{\rm anc}$ is constructed as a random matrix and $r= 2,000$.
Then, we evaluate the following four values,
\begin{align*}
  &\tau_1 = \| \Sigma_2 \|_{\rm F} / \| \Sigma_1 \|_{\rm F}, \\
  &\tau_2 = \| \Sigma_{F2} \|_{\rm F} / \| \Sigma_{F} \|_{\rm F}, \\
  &\tau_3 = \left| \! \left| X F_1 G_1 - \left[ \begin{array}{c}
      X_1 F_1 G_1 \\
      X_2 F_2 G_2 \\
      \vdots \\
      X_c F_c G_c
    \end{array}
  \right] \right| \! \right|_{\rm F} / \| X F_1 G_1 \|_{\rm F}, \\
  &\tau_4 = 1-{\rm NMI}(Y^{\rm test}_{\rm CDA}, Y^{\rm test}_{\rm CA}),
\end{align*}
where, in $\tau_4$, the value ${\rm NMI}(Y^{\rm test}_{\rm CDA}, Y^{\rm test}_{\rm CA}) \in [0,1]$ denotes the normalized mutual information (NMI) between the prediction results of test dataset $X^{\rm test}$ of the collaborative data analysis and the centralized analysis with dimensionality reduction $F_1 G_1$.
\par
All the numerical experiments were performed on Windows 10 Pro, Intel(R) Core(TM) i7-10710U CPU @ 1.10GHz, 16GB RAM using MATLAB2019b.
\par
First, we perform the methods 10 times with $\varepsilon = 0$, that is $\mathcal{R}(F_i) = \mathcal{R}(F_{i'})$, but $F_i \neq F_{i'}$ for $i \neq i'$.
The obtained average values are 
\begin{align*}
  &\tau_1 = 8.42 \times 10^{-16}, \quad
  \tau_2 = 2.24 \times 10^{-16}, \\
  &\tau_3 = 1.44 \times 10^{-13}, \quad
  \tau_4 = 0.00,
\end{align*}
that mean the collaborative data analysis obtains the same result as the centralized analysis with the dimensionality reduction, i.e., $Y^{\rm test}_{\rm CDA} = Y^{\rm test}_{\rm CA}$.
This result supports Theorem~1.
\par
Next, we perform the methods 500 times with random $\varepsilon \in [10^{-2}, 10^{-6}]$ and evaluate correlation coefficients of $\log(\tau_i)$.
Figure~\ref{fig:ccm} shows the scatter plot matrix and correlation coefficients for $\log(\tau_i)$, which demonstrates that the accuracy of the collaborative data analysis regarding $\tau_4$ has a strong correlation between $\tau_1, \tau_2$ and $\tau_3$.
Therefore, from this result, we observed that the accuracy of the collaborative data analysis $\tau_4$ can be evaluated well by $\tau_1$ in practice.
Note that $\tau_1$ can be obtained at the analyst side.
\begin{figure}[t]
\centering
\includegraphics[scale=0.6, bb = 91 252 490 573]{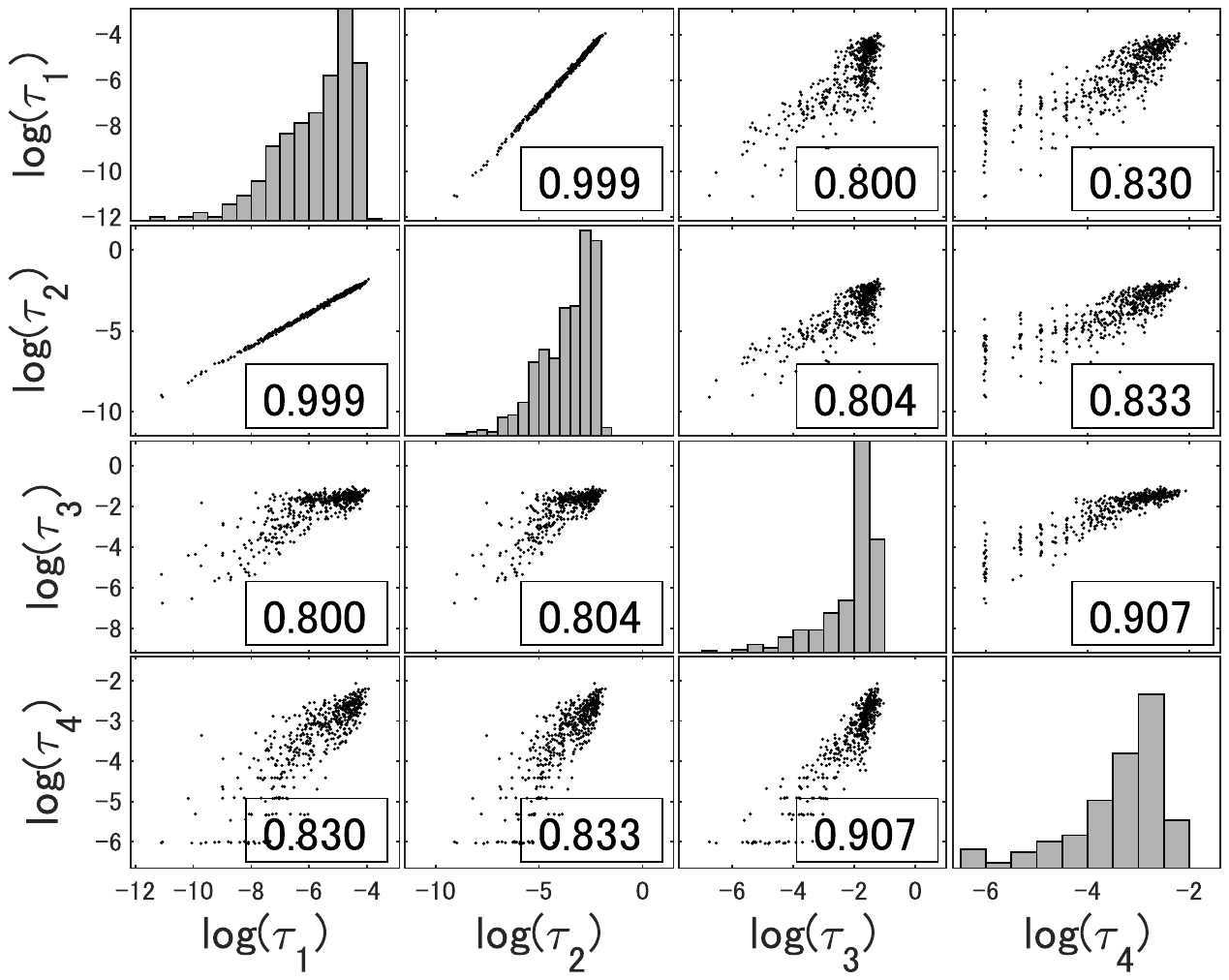}
\caption{Scatter plot matrix and correlation coefficients for $\log(\tau_i)$.}
\label{fig:ccm}
\end{figure}
\par
We can also observed from Figure~\ref{fig:bound} that $\tau_4$ is roughly bounded by 
\begin{equation*}
  \tau_4 \leq c \sqrt{\tau_1} \quad \left( \Leftrightarrow \log(\tau_4) \leq \frac{\log(\tau_1)}{2} + \log(c) \right)
\end{equation*}
with some constant $c$.
Note that $c = \exp(0.5)$ in Figure~\ref{fig:bound}.
%
\begin{figure}[t]
\centering
\includegraphics[scale=0.35, bb = 91 252 490 573]{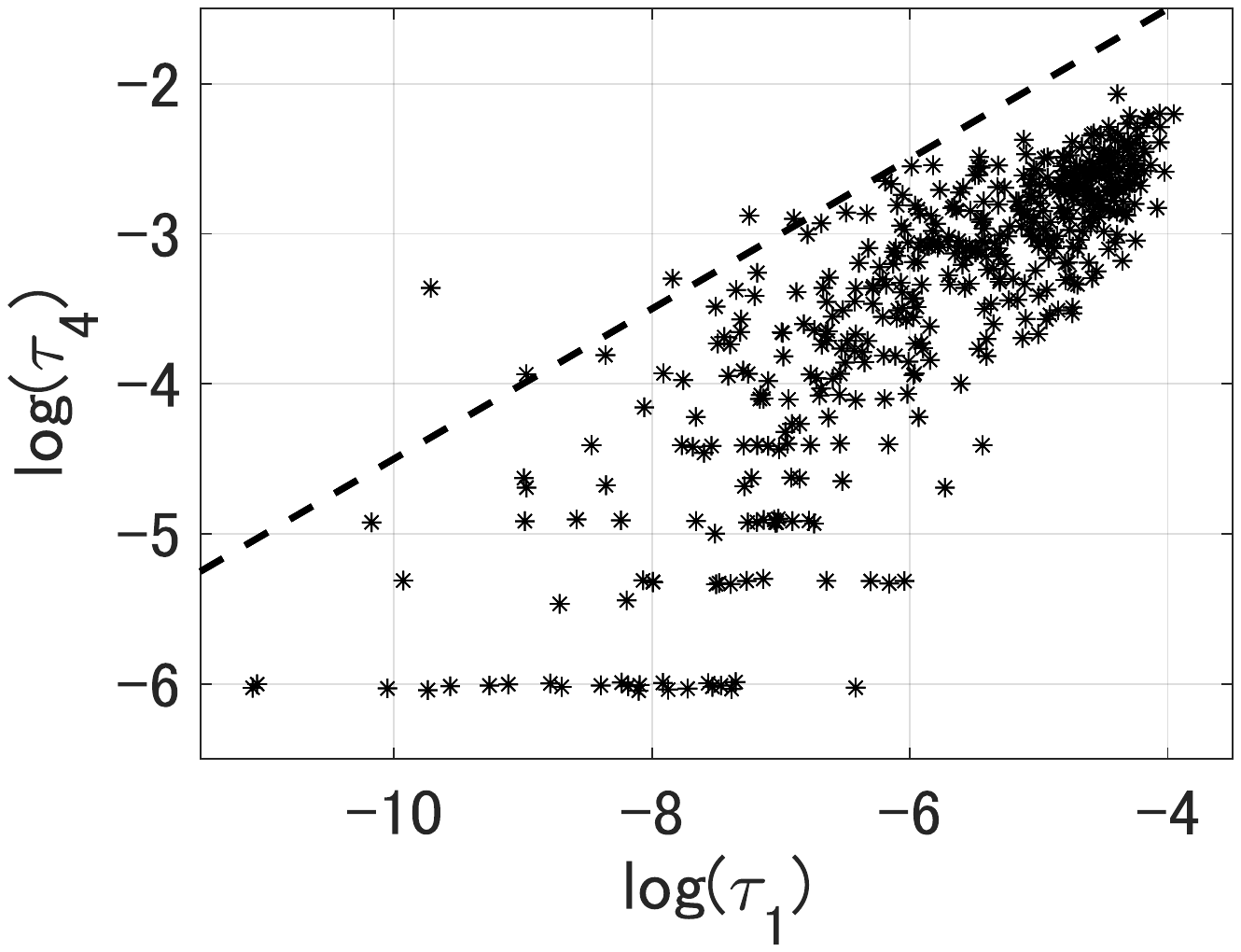}
\caption{Scatter plot of $\log(\tau_4)$ v.s. $\log(\tau_1)$ and its rough bound $\log(\tau_1)/2 + 0.5$.}
\label{fig:bound}
\end{figure}
\subsection{Remarks on accuracy analysis}
From the above analysis, we observed that 
\begin{itemize}
  \item If $F_i$ satisfy the sufficient condition \eqref{eq:cond}, the collaborative data analysis achieves the same result of the centralized analysis with dimensionality reduction.
  \item The accuracy of the collaborative data analysis compared with the centralized analysis with dimensionality reduction can be evaluated by $\| \Sigma_{2} \|_{\rm F}$ in practice.
\end{itemize}
\par
Note that, in order to obtain (approximately) the same predictive results as the centralized analysis with dimensionality reduction, we do not need to use the same mapping functions $f_{i}$, but use different functions with (approximately) the same range space.
\section{Privacy analysis}
For the analysis, this paper considers the privacy of the private data $X_{i}$ of each user in the collaborative data analysis.
Note that any information of the test data $X_{i}^{\rm test}$ does not have to be shared to others; see Algorithm~\ref{alg:proposed}.
\par
Attacks for the data $X_{i}$ can be classified into 
(i) attacks to infer the characteristics of the training data;
(ii) attacks to infer the training data $X_{i}$ itself; and
(iii) attacks to infer whether a data sample is in the training dataset or not, so-called the membership inference attack.
\par
This paper considers the privacy of the data $X_i$ itself, rather than the characteristics of the data.
We also shortly discuss the privacy against the membership inference attack.
%
%
\subsection{Privacy definitions of dimensionality reduction: $\varepsilon$-DR privacy}
Here, we introduce two dimensionality reduction (DR) privacy definitions: $\varepsilon$-DR privacy introduced in \cite{nguyen2019auto} and its variant, to evaluate the degree to which privacy is preserved through dimensionality reduction.
Let $f: {\bm x} \in \mathbb{R}^m \rightarrow \widetilde{\bm x} \in \mathbb{R}^{\widetilde{m}}$ $(m > \widetilde{m})$ be a dimensionality reduced function and $f^\dagger$ be a reconstruction function of $f$.
Then, we evaluate degree of privacy preservation using ${\rm dist}({\bm x}, {\bm x}')$ with a certain distance measure ${\rm dist}(\cdot,\cdot)$, where ${\bm x}' = f^\dagger(f({\bm x}))$.
\par
The function $f$ satisfies the $\varepsilon$-DR privacy regarding the expected value, if we have
\begin{equation}
  E[{\rm dist}({\bm x}, {\bm x}')] \geq \varepsilon_1
  \label{eq:eDR}
\end{equation}
for each i.i.d. input sample ${\bm x}$.
The value $\varepsilon_1$ depends on $f$ and is always larger than $0$ with $m > \widetilde{m}$.
\par
Also, a function $f$ satisfies the $\varepsilon$-DR privacy regarding a sample set $\mathcal{X} = \{ {\bm x}_1, {\bm x}_2, \dots, {\bm x}_n \}$, if we have
\begin{equation}
  \min_{ {\bm x} \in \mathcal{X} } {\rm dist}({\bm x}, {\bm x}') \geq \varepsilon_2.
  \label{eq:worst}
\end{equation}
The value $\varepsilon_2$ depends on $f$ and a set of samples $\mathcal{X}$.
Therefore, since $\varepsilon$ depends on $\mathcal{X}$, $\varepsilon_2$ is possible to be $0$ even if $m > \widetilde{m}$.
\subsection{Attacking scenarios}
In this paper, we consider the following two attacking scenarios: insider and external attacks.
\begin{itemize}
  \item {\it Insider attacks}.
    Here, the users and analyst will strictly follow the strategy, but they try to infer the private data $X_i$.
  \item {\it External attacks}.
    Here, we consider a man-in-the-middle attack scenario where an attacker eavesdrops the information exchanged among users and analyst and infer the private data $X_i$.
\end{itemize}
\subsection{Privacy against the honest-but-curious analyst}
\begin{theorem}
  For the collaborative data analysis, an honest-but-curious analyst cannot infer the private dataset $X_{i}$ of the users, unless analyst does not collude with user(s).
\end{theorem}
\begin{proof}
For a privacy of $X_{i}$ against the honest-but-curious analyst, each user shares the intermediate representations $\widetilde{X}_{i}$ and $\widetilde{X}_{i}^{\rm anc}$ to analyst.
Here, we consider the possibility of recovering $X_i$ from $\widetilde{X}_{i}$ and $\widetilde{X}_{i}^{\rm anc}$.
\par
If analyst has a mapping function $f_{i}$, analyst can infer $X_{i}$ by solving
\begin{equation*}
  \widetilde{X}_{i} = f_{i} ({X}_{i}).
\end{equation*}
However, the function $f_{i}$ is private in the collaborative data analysis and also cannot be inferred by analyst, because analyst only has the output of $f_{i}$, that is the intermediate representations, but has no input data of $f_{i}$, if analyst does not collude with user(s).
Note that the function $f_i$ is constructed by some dimensionality reduction method with the private data $X_i$.
The function $f_i$ depends on $X_i$, so even if the dimensionality reduction method is identified, $f_i$ itself cannot be inferred.
\par
Thus, analyst cannot obtain the private data $X_{i}$ from the intermediate representations, that proves the theorem.
\end{proof}
\subsection{Privacy against the honest-but-curious users}
\begin{theorem}
  For the collaborative data analysis, an honest-but-curious users cannot infer the private dataset $X_{i}$ of a particular user against collusion of up to $c-2$ users.
\end{theorem}
\begin{proof}
For a privacy of $X_{i}$ against other users, each user shares the local anchor data $X_{i}^{\rm anc}$ to other users.
Also, each user obtain the functions $g_i$ and $h$ from analyst that is constructed based on information of private data of other users.
\par
First, we consider the possibility of recovering the private data from $X_{i}^{\rm anc}$.
The local anchor data does not contain the original $X_{i}$, but may preserve some useful information.
Here, the local anchor data is constructed by users themselves using e.g., GAN and autoencoder with data augmentation.
Users can control the containing information although it may have a trade-off relationship between the performance.
Note that the collaborative data analysis works well even using random anchor data as demonstrated in \cite{imakura2020data,imakura2020collaborative}.
Therefore, users cannot obtain the private information of $X_{i}$ from $X^{\rm anc}_{i}$.
\par
Next, we consider the possibility of recovering the private data from $g_{i}$ and $h$.
When $c-1$ users $i \neq i'$, where the total number of users is $c$, are malicious and they collude together to retrieve information of the private dataset of the remaining (victim) user $i'$, the colluding users have the function $h$ and all $X_i, f_i, g_i$ $(i \neq i')$.
In this case, the function $g_i$ $(i \neq i')$ and $h$ are constructed by $X_i, f_i$ $(i \neq i')$ of the colluding users and $X_{i'}$ of the victim user.
Therefore, the private data $X_{i'}$ of the victim user will be inferred by solving an inverse problem.
On the other hand, when the $c-2$ users collude, the obtained functions $g_i$ and $h$ of the colluding users are affected by remaining two users with equal importance.
Therefore, users cannot infer each private dataset $X_{i}$ of the victim users.
\par
Thus, an honest-but-curious users cannot infer the private dataset of a particular user against collusion of up to $c-2$ users, that proves the theorem.
\end{proof}
\subsection{Privacy against collision of analyst with users}
\begin{theorem}
  If user(s) and analyst collude in the collaborative data analysis, the privacy of $X_{i}$ is preserved regarding $\varepsilon$-DR privacy definitions \eqref{eq:eDR} and \eqref{eq:worst} of each $f_{i}$.
\end{theorem}
\begin{proof}
  If user(s) and analyst collude, then they can obtain both the input $X^{\rm anc}$ and output $\widetilde{X}_{i}^{\rm anc}$ of $f_{i}$.
  In this case, they can infer $f_{i}$ satisfying
  \begin{equation*}
    \widetilde{X}_{i}^{\rm anc} = f_{i}(X^{\rm anc}).
  \end{equation*}
  %
  Therefore, using the inferred $f_{i}$, they can infer $X_{i}$ from $\widetilde{X}_{i}$.
  However, since $f_{i}$ is a dimensionality reduced function, that is $m > \widetilde{m}_{i}$, the privacy is still preserved regarding $\varepsilon$-DR privacy definitions \eqref{eq:eDR} and \eqref{eq:worst} of $f_{i}$.
  In other words, the exact data $X_{i}$ cannot be recovered from $\widetilde{X}_{i}$, even using $f_{i}$.
\end{proof}
\subsection{Privacy against the external attacks}
Using secure data transmission protocols such as Transport Layer Security (TLS), in which the transferred information is encrypted using the private key of the involving parties, the collaborative data analysis also protects the private dataset $X_{i}$ against the man-at-the-middle attackers.
Note that, in this case, we do not use secure multi-party computations, but just use encrypted communication for non private data.
\par
In the case that we do not use secure data transmission protocols, the situation is almost the same as the case that users and analyst collude.
That is, man-at-the-middle attackers can infer $f_{i}$ from $X^{\rm anc}$ and $\widetilde{X}_{i}^{\rm anc}$ and can infer $X_{i}$; however, the privacy is still preserved regarding $\varepsilon$-DR privacy definitions \eqref{eq:eDR} and \eqref{eq:worst} of $f_{i}$.
\subsection{Numerical evaluation for privacy analysis}
Here, we provide numerical evaluation of the worst-case privacy analysis, i.e., the situation of Theorem~4.
Let $B_i \in \mathbb{R}^{m \times m_i}$ be a matrix for dimensionality reduction for $X_i$ as
\begin{equation*}
  \widetilde{X}_i = X_i B_i, \quad B_i \in \mathbb{R}^{m \times m_i}.
\end{equation*}
Let $X_i = [{\bm x}_1^{(i)}, {\bm x}_2^{(i)}, \dots, {\bm x}_{n_i}^{(i)}]^{\rm T}$.
Then, if $B_i$ and the center ${\bm \mu}_i \in \mathbb{R}^m$ of dataset $X_i$ are stolen, $X_i$ is inferred by 
\begin{equation*}
  X'_i 
  = [{\bm x}_1^{(i)'}, {\bm x}_2^{(i)'}, \dots, {\bm x}_{n_i}^{(i)'}]^{\rm T} 
  = \widetilde{X}_i B_i^\dagger + {\bm 1}{\bm \mu}_i^{\rm T}(I - B_i B_i^\dagger),
\end{equation*}
where $B^\dagger$ is the pseudo-inverse of $B$ and ${\bm 1} = [1,1, \dots, 1]^{\rm T}$.
As a $\varepsilon$-DR privacy regarding a sample set \eqref{eq:worst}, we set
\begin{equation}
  \min_{ {\bm x} \in \mathcal{X} } {\rm dist}({\bm x}, {\bm x}') =
  \min_{i,j} \frac{ \| {\bm x}_j^{(i)} - {\bm x}_j^{(i)'} \|_2 }{ \| {\bm x}_j^{(i)} \|_2 }.
  \label{eq:edr}
\end{equation}
Then, we use a down-sampling technique which removes training data samples satisfying \eqref{eq:edr} $< \varepsilon$ by changing $\varepsilon$ and evaluate a trade-off relationship between \eqref{eq:edr} and a prediction accuracy of the collaborative data analysis.
\par
We used MNIST again.
The dimensionality reduction matrix $B_i$ is constructed by PCA using each $X_i$.
We set $c = 10, n_i = 100$ and $m_i = 25$ for parameters.
Other settings of numerical evaluation are the same as used in the numerical evaluation for accuracy analysis.
%
\begin{table}[t]
  \caption{Trade-off relationship between $\varepsilon$-DR privacy and prediction accuracy.}
  \label{table:result}
\begin{center}
\begin{tabular}{cccc} 
\toprule
\multicolumn{1}{c}{down-sampling} & \multicolumn{1}{c}{min} & \multicolumn{1}{c}{Ave.} & \multicolumn{1}{c}{Ave.} \\
\multicolumn{1}{c}{parameter $\varepsilon$} & \multicolumn{1}{c}{$\varepsilon$-DR \eqref{eq:edr} } & \multicolumn{1}{c}{\# of samples} & \multicolumn{1}{c}{ACC} \\ \hline
$0.0   \phantom{001}$  & $7.36 \times 10^{-6}$ & $\phantom{0} 100.00$ & $92.8$  \\
$0.0001             $  & $2.07 \times 10^{-4}$ & $\phantom{00} 99.97 $ & $92.8$  \\
$0.001 \phantom{0}  $  & $1.01 \times 10^{-3}$ & $\phantom{00} 99.75 $ & $92.8$  \\
$0.01  \phantom{00} $  & $1.00 \times 10^{-2}$ & $\phantom{00} 97.50 $ & $92.8$  \\
$0.1   \phantom{000}$  & $1.00 \times 10^{-1}$ & $\phantom{00} 76.85 $ & $91.7$  \\
$0.2   \phantom{000}$  & $2.00 \times 10^{-1}$ & $\phantom{00} 56.43 $ & $90.3$  \\
$0.3   \phantom{000}$  & $3.00 \times 10^{-1}$ & $\phantom{00} 39.35 $ & $88.7$  \\
$0.4   \phantom{000}$  & $4.00 \times 10^{-1}$ & $\phantom{00} 25.84 $ & $86.1$  \\
$0.5   \phantom{000}$  & $5.00 \times 10^{-1}$ & $\phantom{00} 15.98 $ & $81.4$  \\  \hline
\multicolumn{2}{l}{Centralized analysis} & $1000.00$ & $93.6$ \\
\multicolumn{2}{l}{Individual analysis}  & $\phantom{0} 100.00$ & $75.5$ \\
\bottomrule
\end{tabular}
\end{center}
\end{table}
\par
Table~\ref{table:result} shows the trade-off relationship between a minimum $\varepsilon$-DR privacy \eqref{eq:edr}, average number of samples after the down-sampling technique in each party, and average of prediction accuracy (ACC) of 10 trials.
We also show averages of ACC for the centralized and individual analyses.
\par
This result shows that, by the down-sampling technique, we can increase the value of $\varepsilon$-DR privacy \eqref{eq:edr} without loss of ACC; see the case of $\varepsilon = 10^{-2}$.
Also, if the predictive accuracy is allowed to decrease slightly, it can take a larger value of $\varepsilon$-DR privacy \eqref{eq:edr}; see the case of $\varepsilon = 0.2$.
These results mean that a small number of samples significantly reduce the values of $\varepsilon$-DR privacy \eqref{eq:edr}, while these samples do not significantly affect the predictive results.
Additionally, even with a larger $\varepsilon$, e.g., $\varepsilon = 0.5$, the predictive accuracy (ACC) of the collaborative data analysis is still higher than that of the individual analysis.
\subsection{Remarks on privacy analysis}
The collaborative data analysis has the following double privacy layer for protection of the private data $X_{i}$.
\begin{itemize}
  \item No one can have the private data $X_{i}$ because $f_{i}$ is private under the protocol (Theorems~2 and 3).
  \item Even if $f_{i}$ is stolen, the private data $X_{i}$ is still protected regarding $\varepsilon$-DR privacy definitions \eqref{eq:eDR} and \eqref{eq:worst} (Theorem~4).
\end{itemize}
\par
For satisfying $\varepsilon$-DR privacy definitions \eqref{eq:eDR} and \eqref{eq:worst} with certain quantities $\varepsilon_1 > 0$ and $\varepsilon_2 > 0$, we need to pay attention to the construction of $f_{i}$.
Dimensionality reduction method satisfying $\varepsilon$-DR privacy \eqref{eq:eDR} with a given $\varepsilon_1 > 0$ has been proposed in \cite{nguyen2019auto}.
For satisfying $\varepsilon$-DR privacy \eqref{eq:worst} with a given $\varepsilon_2 > 0$, we can use a down-sampling technique, which removes data samples satisfying ${\rm dist}({\bm x}, {\bm x}') < \varepsilon_2$, or a constrained dimensionality reduction method, which adds \eqref{eq:worst} as a constrain in optimization.
\par
We also observed from our numerical evaluation that the down-sampling technique can take a larger value of $\varepsilon$-DR privacy \eqref{eq:edr} with small decreasing of ACC.
\par
Here, we also shortly discuss the privacy against the membership inference attack.
The membership inference attacks involve constructing multiple reference datasets and observing the change in the output of the constructed model according to the presence or absence of the target data samples.
Therefore, they are only feasible in scenarios when individual dimensionality reduction functions $f_i$ are leaked in the collaborative data analysis.
To secure data collaboration from the collusion of users in applications where membership inference is a concern, specialized dimensionality reduction algorithms providing Differential Privacy, such as in \cite{jiang2013differential}, can be also considered. 

%
\section{Conclusions}
\label{sec:conclusion}
%
In this paper, we analyzed the accuracy and privacy of a non-model sharing-type federated learning, so-called collaborative data analysis.
\par
From the accuracy analysis, we provided the sufficient condition \eqref{eq:cond} for equivalence of the collaborative data analysis and the centralized analysis with dimensionality reduction.
We also provided a criteria $\tau_1$ for evaluating accuracy of the collaborative data analysis and numerically evaluated them.
\par
From the privacy analysis, we proved that, in the collaborative data analysis, the privacy of the private dataset is preserved based on a double secureness against insider and external attacking scenarios.
We also evaluated the trade-off relationship of privacy and accuracy and showed that the down-sampling technique can take a larger value of $\varepsilon$-DR privacy \eqref{eq:edr} with small decreasing of prediction accuracy.
\par
In the future, we will further analyze the accuracy and privacy of the collaborative data analysis for more complicated situations, e.g, usage of nonlinear dimensionality reduction function and the case of vertical and horizontal data distribution, with numerical evaluation in real-world problems.
\section{Acknowledgements}
This work was supported in part by the New Energy and Industrial Technology Development Organization (NEDO).
\bibliographystyle{elsart-num-sort}
\bibliography{aaai,mybibfile}
\end{document}